\documentclass[twocolumn, journal]{IEEEtran}
\usepackage{graphicx}
\usepackage{epstopdf}
\usepackage{amsthm}
\theoremstyle{definition}
\usepackage{epsfig}
\usepackage{latexsym}
\usepackage{amsfonts}
\usepackage{here}
\usepackage{rawfonts}
\usepackage{calc}
\usepackage{capitalgreekitalic}
\usepackage{url}
\usepackage{enumerate}
\usepackage{color}
\usepackage[tbtags]{amsmath}
\usepackage{amssymb}
\usepackage{upref}
\usepackage{epic,eepic}
\usepackage{times}
\usepackage{dsfont}
\usepackage{comment}
\usepackage{cite}
\usepackage{bbm}
\usepackage{bm}
\usepackage{amsmath}

\allowdisplaybreaks[4]

\newtheorem{theorem}{\bf Theorem}

\newtheorem{lemma}{\bf Lemma}

\usepackage{dsfont}

\newcounter{step}
\newlength{\totlinewidth}
  {\end{list}%
  \rule{\linewidth}{1pt}}
\newcounter{substep}

  {\end{list}}

\newlength{\aligntop}
\newlength{\alignbot}
 \makeatletter
\renewenvironment{align}{%
  \vspace{\aligntop}
  \start@align\@ne\st@rredfalse\m@ne
}{%
  \math@cr \black@\totwidth@
  \egroup
  \ifingather@
    \restorealignstate@
    \egroup
    \nonumber
    \ifnum0=`{\fi\iffalse}\fi
  \else
    $$%
  \fi
  \ignorespacesafterend%
  \vspace{\alignbot}\par\noindent
} \makeatother

\IEEEoverridecommandlockouts

\usepackage{algorithm}
\usepackage{algorithmic}

\usepackage{booktabs}
\usepackage{multirow}
\usepackage[table]{xcolor} 

\usepackage{subfigure}

\begin{document}
\title{Wireless Federated Multi-Task LLM Fine-Tuning via Sparse-and-Orthogonal LoRA}
\author{
{Nuocheng Yang,} \emph{Student Member, IEEE}, 
{Sihua Wang,} \emph{Member, IEEE},  
\\ {Ouwen Huan,} \emph{Student Member, IEEE}, 
{Mingzhe Chen}, \emph{Senior Member, IEEE},
\\{Tony Q. S. Quek}, \emph{Fellow, IEEE}, and {Changchuan Yin}, \emph{Senior Member, IEEE}.
\thanks{N. Yang, S. Wang, O. Huan, and C. Yin are with the Beijing Laboratory of Advanced Information Network, and the Beijing Key Laboratory of Network System Architecture and Convergence, Beijing University of Posts and Telecommunications, Beijing 100876, China (emails: \{yangnuocheng, sihuawang, ouwenh, ccyin\}@bupt.edu.cn).}
\thanks{M. Chen is with the Department of Electrical and Computer Engineering and Institute for Data Science and Computing, University of Miami, Coral Gables, FL, 33146 USA (email: \protect\url{mingzhe.chen@miami.edu}).}
\thanks{Tony Q. S. Quek with the Information Systems Technology and Design Pillar, Singapore University of Technology and Design, 487372, Singapore (email: \protect\url{tonyquek@sutd.edu.sg}).}
}
\maketitle

\thispagestyle{empty}
\pagestyle{empty}

\begin{abstract}
Decentralized federated learning (DFL) based on low-rank adaptation (LoRA) enables mobile devices with multi-task datasets to collaboratively fine-tune a large language model (LLM) by exchanging locally updated parameters with a subset of neighboring devices via wireless connections for knowledge integration.
However, directly aggregating parameters fine-tuned on heterogeneous datasets induces three primary issues across the DFL life-cycle: 
(i) \textit{catastrophic knowledge forgetting during fine-tuning process}, arising from conflicting update directions caused by data heterogeneity; 
(ii) \textit{inefficient communication and convergence during model aggregation process}, due to bandwidth-intensive redundant model transmissions; and 
(iii) \textit{multi-task knowledge interference during inference process}, resulting from incompatible knowledge representations coexistence during inference.
To address these issues in a fully decentralized scenario, we first propose a sparse-and-orthogonal LoRA that ensures orthogonality between model updates to eliminate direction conflicts during fine-tuning.
Then, we analyze how device connection topology affects multi-task performance, prompting a cluster-based topology design during aggregation.
Finally, we propose an implicit mixture of experts (MoE) mechanism to avoid the coexistence of incompatible knowledge during inference.
Simulation results demonstrate that the proposed approach effectively reduces communication resource consumption by up to $73\%$ and enhances average performance by $5\%$ compared with the traditional LoRA method.
\end{abstract}


\IEEEpeerreviewmaketitle
\vspace{-0.3cm}
\section{Introduction}
Large language models (LLMs) have demonstrated their comprehension and reasoning capabilities across general tasks, surpassing traditional artificial intelligence (AI) models tailored to a single task \cite{MobileLLaMA,TowardsCommunicationEfficientMultiAgentCooperationsReinforcementLearningandLLM, LargeLanguageModelsforNetworking, TokenLevelLLMCollaboration}.
Despite their capabilities in general tasks, adapting to specific downstream applications requires continuous fine-tuning based on the low-rank adaptation (LoRA) method \cite{LoRA} and multi-task dataset distributed across devices \cite{FederatedLargeLanguageModel}.
To fine-tune LLMs on multi-task datasets, traditional centralized methods usually require a server for raw data or model updates collection from massive devices through wireless connections, which introduces a transmission bottleneck and privacy concerns.
To address these issues, a decentralized federated learning (DFL) framework was proposed, which leverages devices to collaboratively fine-tune a global model \cite{WirelessCommunicationsCollaborativeFederatedLearning}. 
In this paradigm, devices exchange locally fine-tuned model parameters via device-to-device connections for aggregation, thereby eliminating the dependency on centralized coordination.

However, directly aggregating model updates fine-tuned on heterogeneous downstream tasks' datasets from all available neighbors in DFL may lead to unstable and inefficient convergence, which is primarily caused by three issues across the DFL life-cycle: \textit{catastrophic knowledge forgetting}, \textit{inefficient communication and convergence}, and \textit{multi-task knowledge interference} \cite{NatureHumanBehaviour,DiameterConstrainedTopology,OMoE}.
\textcolor{black}{Firstly, \textit{catastrophic knowledge forgetting} stems from the directional conflict between local and aggregated neighbor's updates caused by data heterogeneity, which leads to the forgetting of prior knowledge during the fine-tuning process \cite{NatureHumanBehaviour}.}
Secondly, \textit{inefficient communication and convergence} caused by the redundant inter-device model exchanges, which force the aggregation of conflicting updates that impede convergence, leading to catastrophic forgetting and wasted bandwidth during aggregation process \cite{DiameterConstrainedTopology}.
\textcolor{black}{Thirdly, \textit{multi-task knowledge interference} induced by merging multiple task-specific updates into a single aggregated model, which blurs task boundaries and leads to output distortion during the inference process \cite{OMoE}.}

To mitigate the \textit{catastrophic knowledge forgetting}, the existing works have proposed the orthogonality subspace updates method for eliminating direction conflict between newly aggregated model updates and the prior model updates \cite{OrthogonalGradientDescentforContinualLearning,wang-etal-2023-orthogonal,AdaptivePlasticityImprovementforContinualLearning,yang2024parameter}.
This method ensures that the model updates fine-tuned on heterogeneous downstream tasks are projected onto orthogonal subspaces, thereby confining interference to distinct dimensions and effectively mitigating forgetting.
However, in these studies, the subspaces of model update evolve continuously during the fine-tuning process, necessitating a centralized server for model update collection and enforce orthogonality, which is incompatible with the DFL setting.

To address \textit{inefficient communication and convergence}, the relationship between multi-task performance and device connection schemes must be analyzed. 
Nevertheless, existing works \cite{mypaper,ADFlora,LearningtoCollaborate} overlooked the impact of model parameter orthogonality and failed to account for the unique constraints of DFL, where devices must autonomously determine a subset of neighbours for local model-update transmission.

To reduce the \textit{multi-task knowledge interference}, mixture of experts (MoE) methods are employed that regard model parameters fine-tuned on heterogeneous datasets as experts independently and introduce an additional MoE router for expert selection, and hence achieving knowledge isolation during the inference process \cite{FFTMoE,dFLMoE,MoORE}.
Even so, training the MoE router incurs computational overhead that scales with the number of devices, burdening the resource-constrained devices in DFL.
Furthermore, the parameter isolation inherent in MoE design completely hinders the integration of knowledge among devices.

To fill these gaps, in this paper, we develop a collaborative LLM fine-tuning framework in DFL by integrating a sparse-and-orthogonal LoRA method, a task-aware implicit MoE mechanism, and a cluster-based device connection topology design.
Specifically, our key contributions are as follows:
\begin{itemize}
    \item \textbf{Sparse-and-Orthogonal LoRA for Reducing Catastrophic Knowledge Forgetting during the Fine-tuning Process:}
    We decompose parameter updates into a static low-rank projection matrix, sampled from a Gauss-sampled, and a low-rank, sparse-activated expansion matrix.
    Consequently, the model updates fine-tuned on heterogeneous datasets are projected onto mutually orthogonal subspaces; this property is maintained by the fixed projection matrices, which inherit their orthogonality from independent Gaussian sampling.
    Furthermore, the sparse activated expansion matrix is introduced to further mitigate parameter collision while reducing the number of parameters that need to be updated and transmitted.
    \item \textbf{Cluster-based Device Connection Topology Design for Communication-Efficient and DFL Acceleration during the Model Aggregation Process:}
    Based on the proposed sparse-and-orthogonal LoRA, we further propose a novel device connection topology design for efficient model updates transmission in multi-task DFL.
    Specifically, we first analyze how the device connection topology affects the cumulative error of the aggregated LoRA representation, which affects the fine-tuning performance.
    Then, we develop a cluster-based device connection topology design that can accelerate DFL convergence while reducing communication overhead.
    \item \textbf{Implicit MoE Design for Minimizing Multi-task Knowledge Interference during the Inference Process:}
    We propose an implicit MoE mechanism that embeds task-specific information into the static projection matrix, coupled with a top-$k$ expert activation strategy. 
    This method selectively activates a subset of parameters within the aggregated expansion matrix (regarded as experts) that exhibit the highest responsiveness to the input tokens.
    Therefore, the projection matrix embedded with task-specific information can act as an implicit MoE router that activates to the related experts, while filtering out interference from irrelevant experts without an additional training process.
\end{itemize}

An extensive experiment shows that the proposed method can effectively reduce communication overhead by up to $73\%$ while improve overall LLM performance by up to $5\%$ compared with the traditional LoRA method.
\section{Related Works}
\subsection{Decentralized Federated Learning based on LoRA}
LoRA is one of the widely used parameter-efficient fine-tuning techniques that enables resource-constrained devices to adapt pre-trained LLMs with task-specific knowledge locally \cite{LoRA}. 
In particular, LoRA integrates trainable adapters into each pre-trained LLM layer.
Each adapter is characterized by a pair of low-rank matrices, named the projection matrix $\boldsymbol{A}$ and expansion matrix $\boldsymbol{B}$, respectively. 
During fine-tuning, only these low-rank matrices are updated while the pre-trained LLM remain frozen.
Since low-rank matrices contain significantly fewer parameters than the pre-trained LLM, this approach drastically reduces the computation overhead of fine-tuning.

Furthermore, multi-task fine-tuning of LLM traditionally relies on a centralized server to collect raw data or updated models from all participating devices, which incurs a prohibitive transmission bottleneck and raises significant privacy concerns
\cite{FederatedFineTuningofLLMs,FedQLoRA,FedALoRA}.
To fill these gaps without centralization constraints, DFL, which enables devices to exchange updated local matrices for integration, has been proposed.
To further accommodate the heterogeneous computational and communication resources of individual devices in DFL, strategies such as adaptive rank selection \cite{Communication-EfficientWirelessFederatedFine-TuningforLarge-ScaleAIModels}, adaptive trainable layers selection \cite{FedHeLLo}, and adaptive adapters quantization strategy \cite{FedQuad} have been explored.
However, directly aggregating model updates on heterogeneous datasets poses two key bottlenecks in DFL.
First, due to the datasets' heterogeneity, the adapters exhibit divergent update directions; these updates tend to offset the datasets' heterogeneity, resulting in \textit{catastrophic knowledge forgetting}.
Second, the coexisting adapters trained based on heterogeneous datasets in the aggregated adapter may cause \textit{multitask knowledge interference} due to the competing objectives of different downstream tasks.

\subsection{Interference-free with Orthogonality Constraints}
To address the \textit{catastrophic knowledge forgetting} problem caused by interference between adapters with divergent update directions, several studies have proposed a model orthogonality-based approach \cite{OrthogonalGradientDescentforContinualLearning,wang-etal-2023-orthogonal,AdaptivePlasticityImprovementforContinualLearning,yang2024parameter}.
The authors in \cite{OrthogonalGradientDescentforContinualLearning} first attributed \textit{catastrophic knowledge forgetting} to conflicts between model updates.
Then, they proposed an orthogonal gradient descent (OGD) method to prevent conflicts in model update directions by introducing a direction orthogonal constraint into the local objective function.
Note that, to ensure the mutual orthogonality between model update directions, all devices' model updates must be collected and constrained simultaneously. 
Inspired by \cite{OrthogonalGradientDescentforContinualLearning}, the authors in \cite{wang-etal-2023-orthogonal} first introduced OGD into multi-task LLM fine-tuning, and the authors in \cite{AdaptivePlasticityImprovementforContinualLearning} expanded it to the federated learning scenario, where the centralized server collects model updates from devices and ensures orthogonality.
Beyond maintaining orthogonality via explicit constraints, the authors in \cite{yang2024parameter} demonstrated that parameter orthogonality can also be achieved by ensuring non-overlapping updates across multi-tasks.
However, these works \cite{OrthogonalGradientDescentforContinualLearning,wang-etal-2023-orthogonal,AdaptivePlasticityImprovementforContinualLearning,yang2024parameter} can only be employed with a centralized server to coordinate the disruption of the orthogonality caused by local fine-tuning and non-parameter collisions between updates.
On the other hand, the absence of a central server in DFL, coupled with constrained communication resources, prevents devices from accessing all participants' updates to ensure orthogonality \cite{WirelessCommunicationsCollaborativeFederatedLearning}. 

\subsection{Multi-task Knowledge Isolation based on MoE}
To address \textit{multi-task knowledge interference} caused by multiple knowledge coexisting in the aggregated adapter, several works \cite{FFTMoE,dFLMoE,MoORE} have focused on MoE based approaches.
These methods treat each adapter fine-tuned on heterogeneous datasets as an expert and introduce a router to assign inputs to the most appropriate subset of experts.
The authors in \cite{FFTMoE} initialized FL with MoE-based adapters, where each device trains a lightweight gating network to selectively activate a personalized subset of experts to enhance the multi-task fine-tuning performance.
Building upon a similar intuition, the authors in \cite{dFLMoE} employed an attention-based MoE router to manage input dispatching within a DFL framework.
However, in these works \cite{FFTMoE,dFLMoE}, both the number of experts and the router's training overhead scale linearly with the increasing number of devices, leading to significant scalability challenges in large-scale networks.
To eliminate the drawbacks caused by an additional MoE router, the authors in \cite{MoORE} employed singular value decomposition (SVD) to decompose adapters into experts, which are dynamically engaged according to their singular values.
These works \cite{FFTMoE,dFLMoE,MoORE} also face challenges in aggregating experts into the converged MoE models, particularly in balancing expert utilization and designing effective mixing strategies. 
Additionally, due to limited communication capability, devices can only exchange updates with a few neighbors, necessitating an efficient connection topology to optimize expert aggregation and accelerate MoE convergence.

\subsection{Communication-Efficient Topology for DFL Convergence Acceleration}
Due to the limited communication resources of each device in the DFL, they must select a subset of neighbors for exchanging the locally updated adapters, which also impacts communication overhead and DFL convergence speed.
An inefficient device connection scheme not only incurs substantial communication overhead but also prevents the aggregated model from converging, due to \textit{catastrophic knowledge forgetting} and \textit{multi-task knowledge interference} problems.
To fill these gaps, the authors in \cite{mypaper} analyzed how device connection topology affects the model transmission consumption and the convergence speed in DFL with single-task and proposed a decentralized graph neural network based topology design approach.
The authors in \cite{ADFlora} proved that the device connection topology also impacts the knowledge aggregation order in single-task learning, and they optimize the device connection topology by considering graph connectivity and model alignment.
The authors in \cite{LearningtoCollaborate} proposed a device topology design method based on model distillation and feature mapping method to simultaneously enhance DFL aggregation performance and ensure security.
However, these works \cite{mypaper,ADFlora,LearningtoCollaborate} still suffer from \textit{catastrophic knowledge forgetting} and \textit{multi-task knowledge interference}.
Therefore, it is desirable to design a new collaborative LLM fine-tuning method in the DFL scenario by jointly considering the \textit{catastrophic knowledge forgetting}, \textit{multi-task knowledge interference}, and the \textit{substantial communication resource
consumption} under DFL resource constraints.

\section{System Setup and Problem Formulation}
Consider a distributed wireless network consisting of a set $\mathcal{M}$ of $M$ devices that collaboratively train an LLM with $L$ layers that can adapt to multiple downstream tasks.
We assume that each device $i \in \mathcal{M}$ has a local dataset $\mathcal{D}_i$ for a unique type of downstream task $K_i$, containing $N_i$ data samples with $N=\sum_{i = 1}^{M} N_i$ being the total number of data samples across devices. 
Each data sample $n$ consists of an input feature vector $\boldsymbol{x}_{i,n} \in \mathbb{R}^{N_{\textrm{I}}\times 1}$ and a corresponding label vector $\boldsymbol{y}_{i,n} \in \mathbb{R}^{N_{\textrm{O}}\times 1}$. 
In particular, the local objective function of device $i$ is given by
\begin{equation}\label{eq:LocalLoss}
    F\left(\boldsymbol{W}_0, \boldsymbol{w}_{i,t},\mathcal{D}_{i}\right)
    \!=\!\frac{1}{N_i}\sum\limits_{n=1}^{N_i}\!f\!\left(\phi\left(\boldsymbol{W}_0, \boldsymbol{w}_{i,t}, \boldsymbol{x}_{i,n}\right),\boldsymbol{y}_{i,n}\right),
\end{equation}
where $\boldsymbol{W}_0=\left\{\boldsymbol{w}_{0,1}, \cdots, \boldsymbol{w}_{0,L}\right\}$ denotes the fixed pre-trained LLM parameters set, where $\boldsymbol{w}_{0,l} \in \mathbb{R}^{d_l \times k_l}$ being the parameters of $l$-th pre-trained LLM layer.
Similarly, $\boldsymbol{w}_{i,t}=\left\{\boldsymbol{w}_{i,1,t}, \cdots, \boldsymbol{w}_{i,L,t}\right\}$ is the trainable adapter parameters set of device $i$ at iteration $t$ , where $\boldsymbol{w}_{i,l,t} \in \mathbb{R}^{d_l \times k_l}$ being the parameters of $l$-th fine-tuned LLM layer.
$\phi\left(\boldsymbol{W}_0,\boldsymbol{w}_{i,t}, \boldsymbol{x}_{i,n}\right)$ denotes the output of fine-tuned LLM, and $f\left(\phi\left(\boldsymbol{W}_0, \boldsymbol{w}_{i,t}, \boldsymbol{x}_{i,n}\right),\boldsymbol{y}_{i,n}\right)$ is the loss function that measures the difference between the output $\phi\left(\boldsymbol{W}_0,\boldsymbol{w}_{i,t}, \boldsymbol{x}_{i,n}\right)$ and label $\boldsymbol{y}_{i,n}$. 

\subsection{DFL Local Fine-tuning Process}
\begin{figure}[t]
\centering
\includegraphics[width=7cm]{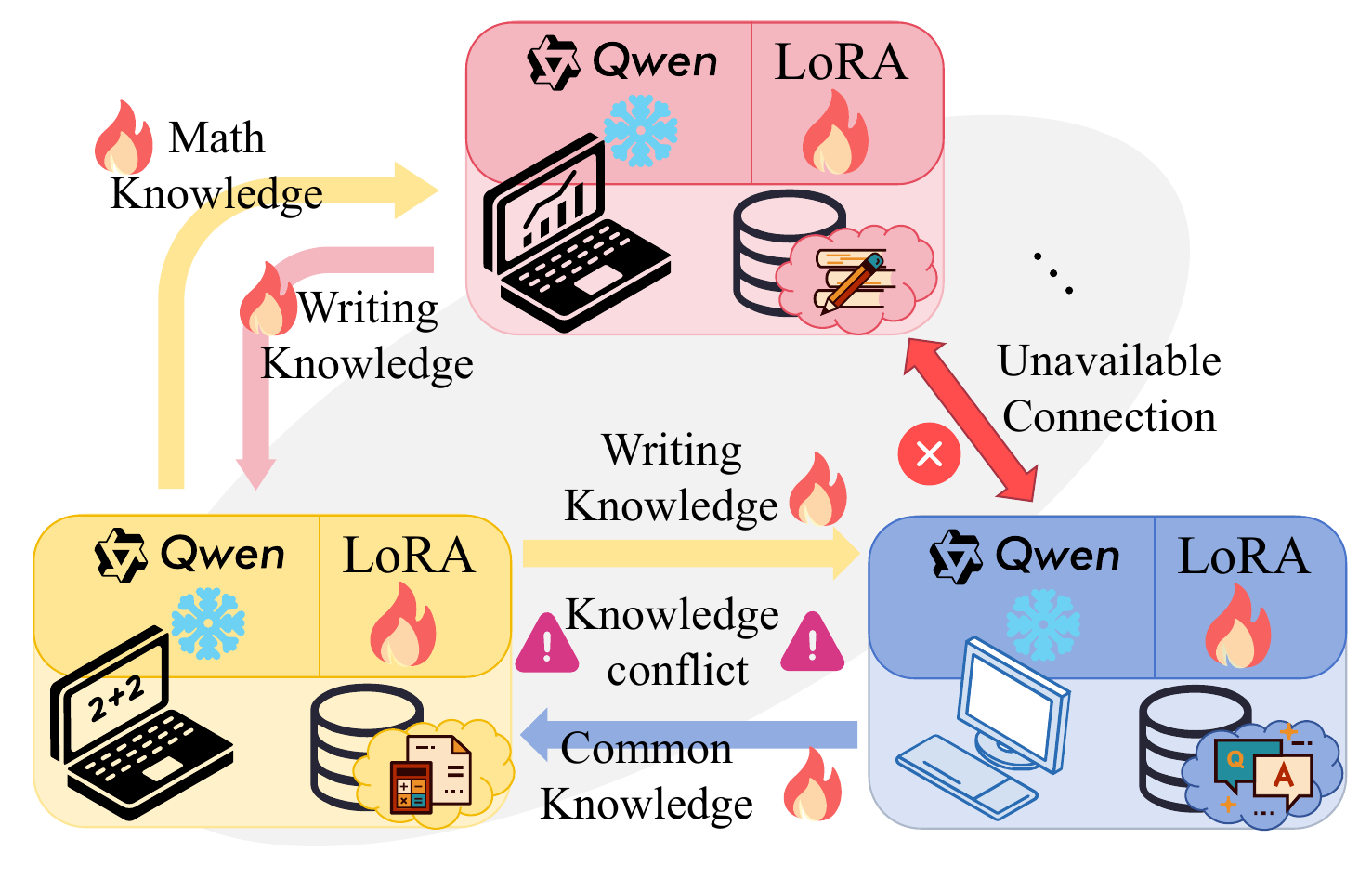}
\vspace{-0.5cm}
\caption{Illustration of the considered multi-task DFL framework.}
\centering
\vspace{-0.4cm}
\label{fig1}
\end{figure}
To reduce the computation overhead caused by full parameter updating, LoRA is employed to decompose adapter $\boldsymbol{w}_{i,l,t} \in \mathbb{R}^{d_l \times k_l}$ into two low-rank matrices $\boldsymbol{B}_{i,l,t}\in \mathbb{R}^{d_l \times r_l}$ and $\boldsymbol{A}_{i,l,t}\in \mathbb{R}^{r_l\times k_l}$ in each device with which has fewer parameters, which is given by
\begin{equation}\label{eq:localLoRA}
    \boldsymbol{w}_{i,l,t}=\boldsymbol{B}_{i,l,t}\boldsymbol{A}_{i,l,t},
\end{equation}
where $\boldsymbol{B}_{i,l,t}$ is called expansion matrix and $\boldsymbol{A}_{i,l,t}$ is called projection matrix, respectively.
Expansion matrix and projection matrix share the same rank $r_{l} \ll \min\left(d_{l},k_{l}\right)$.

To further reduce the computational overhead for local LoRA parameter updating, a sparsity-activation-based method is employed, which activates only a subset of parameters in $\boldsymbol{B}_{i,l,t}$ during local fine-tuning.
The sparsification mask for $\boldsymbol{B}_{i,l,t}$ is $\boldsymbol{M}^{B}_{i,l,t}\in\{0,1\}^{d_l \times r_l}$ where $\left[\boldsymbol{M}^{B}_{i,l,t}\right]_{x,y}=1$ implies the $\{x,y\}$-th parameter in $\boldsymbol{B}_{i,l,t}$ is non-zero and trainable during updating, and $\left[\boldsymbol{M}^{B}_{i,l,t}\right]_{x,y}=0$, otherwise.
Thus, the adapter of device $i$ in (\ref{eq:localLoRA}) can be rewritten as 
\begin{equation}
  \begin{aligned}
    \boldsymbol{w}_{i,l,t}=\left(\boldsymbol{B}_{i,l,t} \odot \boldsymbol{M}^{B}_{i,l,t}\right)\boldsymbol{A}_{i,l,t},
  \end{aligned}\label{mask_LoRA}
\end{equation}
where $\odot$ is the element-wise multiplication function.
Then, we define the sparsity rate of $\boldsymbol{B}_{i,l,t}$, which represents the average activation probability on each parameter position of $\boldsymbol{B}_{i,l,t}$ during fine-tuning as 
\begin{equation}
  \begin{aligned}
    s_{i,l,t}=\frac{||\boldsymbol{M}^{B}_{i,l,t}||}{d_l \times r_l},
  \end{aligned}
\end{equation}
where $||\boldsymbol{M}^{B}_{i,l,t}|| = \sum\limits_{x=1}^{d_l}\sum\limits_{y=1}^{r_l}  \left[\boldsymbol{M}^{B}_{i,l,t}\right]_{x,y}$ is the number of the activated parameter that device $i$ will transmit to its neighbors. 

Generally, in each iteration $t$, each device $i$ performs a local model update based on the local dataset $\mathcal{D}_{i}$ as follows
\begin{align}
    \boldsymbol{A}^{\prime}_{i,l,t} &= \boldsymbol{A}_{i,l,t} - \eta \nabla_{\boldsymbol{A}_{i,l,t}} F\left(\boldsymbol{W}_0, \boldsymbol{w}_{i,t},\mathcal{D}_{i}\right), \label{eq:updateA} \\
    \boldsymbol{B}^{\prime}_{i,l,t} &= \boldsymbol{B}_{i,l,t} - \eta \nabla_{\boldsymbol{B}_{i,l,t}\odot \boldsymbol{M}^{B}_{i,l,t}} F\left(\boldsymbol{W}_0, \boldsymbol{w}_{i,t},\mathcal{D}_{i}\right), \label{eq:updateB}
\end{align}
where $\eta$ is the learning rate, $\nabla_{\boldsymbol{A}_{i,l,t}} F\left(\boldsymbol{W}_0, \boldsymbol{w}_{i,t},\mathcal{D}_{i}\right)$ and $\nabla_{\boldsymbol{B}_{i,l,t}\odot \boldsymbol{M}^{B}_{i,l,t}} F\left(\boldsymbol{W}_0, \boldsymbol{w}_{i,t},\mathcal{D}_{i}\right)$ are the cumulative gradient.

The computation overhead of device $i$ in the local fine-tuning process is given by \cite{AJointLearning}
\begin{equation}\label{computing_amount}
  \begin{aligned}
    c_{i,t}= \frac{\varepsilon Q_{i,t}\left(\boldsymbol{s}_{i,t}\right)N_{i}}{B},
  \end{aligned}
\end{equation}
where $\varepsilon$ donate the computing correlation coefficient related to the device's computing capability, $Q_{i,t}\left(\boldsymbol{s}_{i,t}\right)=\sum_{l=1}^{L}d_l r_l s_{i,l,t}+r_{l}k_{l}$ being the number of parameters that need to be updated. $\boldsymbol{s}_{i,t}=\left[\boldsymbol{s}_{i,1,t},\dots,\boldsymbol{s}_{i,L,t}\right]$ is the sparsity rate matrix.
, and $B\leq B_{\text{max}}$ is the batch size of local fine-tuning and $B_{\text{max}}$ is the maximum batch size.

\subsection{DFL Model Transmission and Aggregation Process}
After local fine-tuning, each device $i$ will transmit the updated LoRA parameters with its neighbors for model aggregation.
Note that due to the limited transmission resources, only a subset of neighboring devices can be selected for model transmission.
We adopt an orthogonal frequency division multiple access (OFDMA) transmission scheme for model transmission.
Let $W$ be the bandwidth that the device can use to transmit model parameter updates $\boldsymbol{w}_{i,t}$ and $p_{i,j,t}$ be the transmit power. 
The data transmission delay of device $i$ transmitting updated model parameters to device $j$ is
\begin{equation}\label{eq:TimeDelay}
  \begin{aligned}
    l_{i,j,t}\!\left(Q_{i,t}\left(\boldsymbol{s}_{i,t}\right),\!\boldsymbol{u}_{i,t},\!\boldsymbol{\varepsilon}_t,\!p_{i,j,t}\right) \!=\! \frac{Q_{i,t}\left(\boldsymbol{s}_{i,t}\right)}{\frac{W}{||\boldsymbol {u}_{i,t}||}{\log}\left(1+{\frac{{p_{i,j,t}{h_{i,j,t}\left(\boldsymbol{\varepsilon}_{t}\right)}}}{{\sigma^2_{\emph{N}}}}}\right)},
  \end{aligned}
\end{equation}
where $\boldsymbol{u}_{i,t} = [u_{i,1,t},\cdots,u_{i,M,t}]$ is a model transmission vector of device $i$ with $u_{i,j,t}=1$ implying that device $i$ will exchange its local model with device $j$ at iteration $t$, and $u_{i,j,t}=0$, otherwise.
$||\boldsymbol{u}_{i,t}|| = \sum_{j \in \mathcal{M}} u_{i,j,t}$ is the number of devices that will exchange the updated parameters to device $i$. 
The subset of neighbors selected by device $i$ with which it transmit its updated parameter is $\mathcal{M}_{i,t}=\left\{j|u_{i,j,t}=1\right\}$.
$h_{i,j,t}=\rho_{i,j,t}d_{i,j,t}^{-2}$ is the channel gain between device $i$ and $j$ with $\rho_{i,j,t}$ being the Rayleigh fading parameters, and $d_{i,j,t}$ is the distance between device $i$ and $j$. The location of each device $i$ at iteration $t$ is captured by a vector $\boldsymbol{\varepsilon}_{i,t} = [\varepsilon_{i,t,1}, \varepsilon_{i,t,2}]$, with $\boldsymbol{\varepsilon}_{t}=\left[\boldsymbol{\varepsilon}_{1,t},...,\boldsymbol{\varepsilon}_{M,t}\right]$ being the location matrix across devices. 
$\sigma^2_{\emph{N}}$ represents the variance of additive white Gaussian noise.

Given the neighbor's update model parameter, in the traditional DFL framework, devices will aggregate the received model into a single model \cite{FedAVG}, which is given by:
\begin{equation}\label{eq:LoRA_AggregationA}
    \begin{aligned}
    \boldsymbol{A}_{i,l,t+1} = \sum_{j \in \mathcal{M}} \alpha_{j} \frac{u_{j,i,t}}{||\boldsymbol {u}_{i,t}||}\boldsymbol{A}^{\prime}_{j,l,t},
    \end{aligned}
\end{equation}
\begin{equation}\label{eq:LoRA_AggregationB}
    \begin{aligned}
    \boldsymbol{B}_{i,l,t+1} = \sum_{j \in \mathcal{M}} \alpha_{j} \frac{u_{j,i,t}}{||\boldsymbol {u}_{i,t}||} \boldsymbol{B}^{\prime}_{j,l,t}, 
    \end{aligned}
\end{equation}
where $\alpha_{j}\in\left[0,1\right]$ is the aggregation weight based on training data volume and device connection density, which satisfy $\sum_{j\in\mathcal{M}}\alpha_{j}=1$.

\subsection{Problem Formulation}
Our goal is to minimize the DFL training loss across various downstream tasks while accounting for transmission delay, communication and computational resource constraints simultaneously.
The optimization problem is formulated as
\begin{equation}\label{eq:max1}
    \mathop {\min }\limits_{\boldsymbol{U}, \boldsymbol{P}, \boldsymbol{S}} \frac{1}{M} \sum\limits_{i=1}^{M}\sum\limits_{j=1}^{M}F\left(\boldsymbol{W}_0, \boldsymbol{w}_{i,t},\mathcal{D}_{j}\right),\\
    \end{equation}
    \begin{align}\label{c1}
    \setlength{\abovedisplayskip}{-20 pt}
    \setlength{\belowdisplayskip}{-20 pt}
    {\rm{s.t.}} \,\, 
    & l_{i,j,t}\left(Q_{i,t}\left(\boldsymbol{s}_{i,t}\right),\boldsymbol{u}_{i,t}, \boldsymbol{\varepsilon}_t, p_{i,j,t}\right) \leqslant \Gamma,\forall i,j \in \mathcal{M},\forall {t} \in \mathcal{T}, \tag{\theequation a} \\
    & \sum_{j\in\mathcal{M}_{i,t}} p_{i,j,t} \leqslant p_{\text{max}}, \forall i \in \mathcal{M}, \forall {t} \in \mathcal{T},\tag{\theequation b} \\
    & c_{i,t} \leqslant c_{i,\text{max}}, \forall i \in \mathcal{M}, \forall {t} \in \mathcal{T},\tag{\theequation c}
\end{align}
where $\boldsymbol{U} = [\boldsymbol{u}_{1},\cdots,\boldsymbol{u}_{T}]^\top$ is the model transmission matrix, $\boldsymbol{P} = [\boldsymbol{p}_{1},\cdots,\boldsymbol{p}_{T}]^\top$ is the transmit power matrix, and $\boldsymbol{S}=[\boldsymbol{s}_{1},\cdots,\boldsymbol{s}_{T}]^\top$ is the sparsity rate matrix where $\boldsymbol{s}_{t}=\left[s_{i,t},\cdots,s_{M,t}\right]$ is the sparsity rate of each layer at iteration $t$. 
$p_{\text{max}}$ is the transmit power constraint. 
$\Gamma$ is the maximum model transmission delay. 
$p_{\text{max}}$ and $c_{i,\text{max}}$ are the maximum transmission power and maximum computing resources, respectively.
(\ref{eq:max1}a) is a constraint on the model transmission delay per iteration, (\ref{eq:max1}b) is the transmit power constraint, and (\ref{eq:max1}c) is the computing resources constraint.

Second, the design of $\boldsymbol{U}$ and sparsity masks is challenging, and their interplay with DFL performance needs to be investigated.
\section{Proposed Methods}
\begin{figure*}[t]
\centering
\includegraphics[width=16cm]{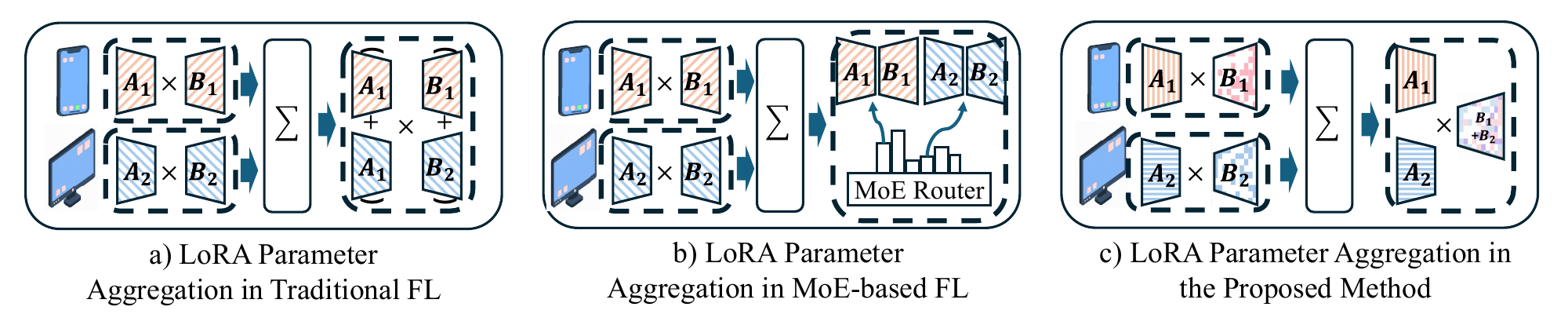}
\caption{Comparison between different LoRA methods during the aggregation process.}
\centering
\end{figure*}
\subsection{Issues of Existing Works}
Based on the modeling above, we identify three primary challenges for collaborative LLM fine-tuning within the DFL life-cycle:
\begin{itemize}
    \item \textbf{Problem 1: Absence of a decentralized mechanism against \textit{catastrophic knowledge forgetting} during fine-tuning process.}
    The existing works introduce orthogonal constraints into the local objective function (\ref{eq:LocalLoss}) to enforce the model updates of different tasks separately in orthogonal subspaces, which is given by
    \begin{equation}
    \label{orthogonallossfunction}
        \begin{aligned}
            F_{\text{orth}}\left(\boldsymbol{W}_0, \boldsymbol{w}_{i,t},\mathcal{D}_{i}\right)
            =&\frac{1}{N_i}\sum\limits_{n=1}^{N_i}\!f\!\left(\phi\left(\boldsymbol{W}_0, \boldsymbol{w}_{i,t}, \boldsymbol{x}_{i,n}\right),\boldsymbol{y}_{i,n}\right) \\&+\sum\limits_{j=1}^{M}\lambda_{i}||\boldsymbol{A}_{i,l,t}\boldsymbol{A}_{i,l,t}^{\top}||^2.
        \end{aligned}
    \end{equation}
    The second term in (\ref{orthogonallossfunction}) is the orthogonal loss function, $\lambda_i$ is the penalty factor. 
    Note that each device needs to collect $\boldsymbol{A}_{i,l,t},\forall i \in \mathcal{M}$ for minimizing $F_{\text{orth}}\left(\boldsymbol{W}_0, \boldsymbol{w}_{i,t},\mathcal{D}_{i}\right)$ contradicts the decentralized setting.
    \item \textbf{Problem 2: Inadequacy of theoretical analysis on the relationship between multi-task performance and device connection scheme design during the aggregation process.}
    The existing DFL works focus on single-task fine-tuning scenario, omitting the analysis of its relationship with multi-task performance.
    \item \textbf{Problem 3: Lack of an computationally inexpensive  method against \textit{multi-task knowledge interference} reduction method during the inference process.}
    For example, consider $\boldsymbol{w}_{i,l,t}$ and $\boldsymbol{w}_{j,l,t}$ are aggregated into a single model based on the traditional method shown in (\ref{eq:LoRA_AggregationA}) and (\ref{eq:LoRA_AggregationB}), which is given by 
    \begin{equation}
      \begin{aligned}
      \boldsymbol{W}^{\prime}_{l,t}&
      \!=\!\boldsymbol{W}_{0,l}\!+\!\left(\alpha_{i}  \boldsymbol{B}^{\prime}_{i,l,t}\!+\!\alpha_{j}  \boldsymbol{B}^{\prime}_{j,l,t}\right)\!\!\left( \alpha_{i} \boldsymbol{A}^{\prime}_{i,l,t}+\alpha_{j} \boldsymbol{A}^{\prime}_{j,l,t}\right).
      \end{aligned}
    \end{equation}
    Let $\boldsymbol{H}_{i,l,\text{input}}$ be the $l$-th layer input for task $i$, derived from dataset $\mathcal{D}_{i}$, the accordingly inference output of the merged model $\boldsymbol{W}^{\prime}_{l,t}$ is given by 
    \begin{equation}
      \begin{aligned}  &\boldsymbol{W}^{\prime}_{l,t}\boldsymbol{H}_{i,l,\text{input}}=\boldsymbol{W}_{0,l}\boldsymbol{H}_{i,l,\text{input}}+\alpha_i^2 \boldsymbol{B}^{\prime}_{i,l,t}\boldsymbol{A}^{\prime}_{i,l,t}\boldsymbol{H}_{i,l,\text{input}}\\
      &\!\!+\!\!\left(\alpha_i\alpha_j \!\!\!\!\sum\limits_{k,l\in\left\{i,j\right\},k\neq l}\!\!\!\!\boldsymbol{B}^{\prime}_{k,l,t}\boldsymbol{A}^{\prime}_{l,l,t}+\alpha_j^2\boldsymbol{B}^{\prime}_{j,l,t}\boldsymbol{A}^{\prime}_{j,l,t}\right)\boldsymbol{H}_{i,l,\text{input}},
      \end{aligned}\label{eq:example}
    \end{equation}
    where the second term is the desired output from the fine-tuned adapter $\boldsymbol{w}_{i,l,t}$.
    The third term of (\ref{eq:example}) can be viewed as a "interference" introduced by the knowledge from other adapters, which needs to be minimized for enhancing LLM performance.
    To achieve this, the existing MoE-based approaches employ a MoE router to determine gating weights for experts (e.g., adapters) during inference process based on their knowledge. 
    The output of the MoE-based model can be given by
    \begin{equation}
      \begin{aligned}         \boldsymbol{W}^{\prime}_{l,\text{MoE}}\boldsymbol{H}_{i,l,\text{input}}\!=\left(\!\boldsymbol{W}_{0,l}\!+\!\sum\limits_{k=1}^{M}\beta_k \boldsymbol{B}^{\prime}_{k,l,t}\boldsymbol{A}^{\prime}_{k,l,t}\right)\boldsymbol{H}_{i,l,\text{input}},
      \end{aligned}
    \end{equation}
    where $\beta_k$ is the gating weight.
    The MoE-based approach introduces a router, incurring additional training overhead that scales with the number of devices and limits its applicability in large-scale deployments.
\end{itemize}
\subsection{Static-and-Orthogonal Projection Matrix and Layer-Wise Sparsity Expansion Matrix Design to Solve Problem 1}
Inspired by \cite{OrthogonalGradientDescentforContinualLearning}, keeping the projection matrix $\boldsymbol{A}_{i,l,t}$ orthogonality between different tasks (i.e., $\sum_{}^{}||\boldsymbol{A}_{i,l,t}\boldsymbol{A}_{i,l,t}^{\top}||^2\approx 0, \forall i,j \in \mathcal{M}, \forall l\in\left[1, L\right]$) can effectively avoid knowledge forgetting during parameter aggregation process.
However, the orthogonality between $\boldsymbol{A}_{i,l,t}$ will be compromised during each local model update, and hence the updated projection matrix from all other devices needs to be collected in real-time to ensure orthogonality, which compromises the communication resources and leads to unstable fine-tuning.

To satisfy the orthogonality between the projection matrix $\boldsymbol{A}_{i,l,t}$ in the DFL scenario without relying on collecting the latest updated parameters from all devices, we regard $\boldsymbol{A}_{i,l,t}$ as a static matrix with entries drawn from a zero-mean, unit-variance Gaussian distribution independently, inspired by LoRI \cite{LoRI}.
In this way, $\boldsymbol{A}_{i,l,t}, \forall i \in \mathcal{M}$ are mutually orthogonal during the fine-tuning process, which is given by [Theorem 3.4, \cite{FlyLoRA}]
\begin{equation}
  \begin{aligned}\label{eq:orgA}
    \sum||\boldsymbol{A}_{i,l,t}^{\top}\boldsymbol{A}_{j,l,t}||^{2}\approx\boldsymbol{0}_{r \times r}, \left[\boldsymbol{A}_{i,l,t}\right]_{x,y} \sim \mathcal{N}\left(0, 1\right), \forall i,j\in\mathcal{M}.
  \end{aligned}
\end{equation}

Since $\boldsymbol{A}_{i,l,t}, \forall i\in \mathcal{M}$ are frozen during fine-tuning and aggregation, we assume they are stored separately, which is different from the standard aggregation step (\ref{eq:LoRA_AggregationA}) in traditional DFL. 
In the proposed framework, each device $i$ transmit its local projection matrix $\boldsymbol{A}_{i,l,t}$ with its neighbors only at the first iteration of parameter exchange, and separately save them into the projection matrix set $\left\{\boldsymbol{A}_{i,l,t}|i\in \mathcal{M}, l\in \left[1,L\right]\right\}$.
Further, we abbreviated static $\boldsymbol{A}_{i,l,t}$ as $\boldsymbol{A}_{i,l}$ in the following section. 
The number of parameters that need to be updated can also be rewritten as 
\begin{align}
Q_{i,t}\left(\boldsymbol{s}_{i,t}\right)=\sum_{l=1}^{L}d_l r_l s_{i,l,t}.
\end{align}
Here the transmission overhead of $\boldsymbol{A}_{i,l}$ is ignored since it only requires to be transmitted once between devices.
Note that keeping $\boldsymbol{A}_{i,l}$ static does not compromise the fine-tuning accuracy compared to traditional DFL where $\boldsymbol{A}_{i,l}$ is continuously updated, as $\boldsymbol{B}_{i,l,t}\boldsymbol{A}_{i,l}$ preserves the same subspace dimensionality.
Furthermore, $\boldsymbol{A}_{i,l}$ only needs to be transmitted once across different pairs of devices, which can also reduce the transmission overhead compared to the traditional approaches where $\boldsymbol{A}_{i,l}$ is transmitted at each iteration.
However, parameter interference during the aggregation of expansion matrix $\boldsymbol{B}_{i,l,t}$ persists even with the implementation of the orthogonal projection matrix.

To further avoid parameter interference between the expansion matrix $\boldsymbol{B}_{i,l,t}$ during aggregation as well as decreasing communication overhead, we propose a sparse-activation mechanism as shown in (\ref{mask_LoRA}). 
This mechanism selectively activates only a subset of parameters in $\boldsymbol{B}_{i,l,t}$ during local fine-tuning, independently based on their communication and computing resources through sparsification mask for $\boldsymbol{M}^{B}_{i,l,t}\in\{0,1\}^{r_l \times k_l}$.
\begin{figure}[t]
\centering
\includegraphics[width=7.5cm]{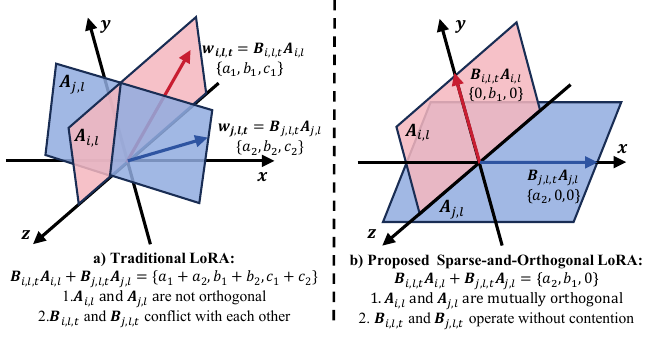}
\caption{Comparison of update conflicts in conventional LoRA and the proposed sparse-and-orthogonal LoRA.}
\centering
\end{figure}
Specifically, to preserve essential model performance, we activate the parameters with the highest $s_{i,l,t}\%$ value within $\boldsymbol{B}_{i,l,t}$ during the initial fine-tuning mini-batch, which can be given by:
\begin{equation}
  \begin{aligned}
    \left[\boldsymbol{M}_{i,l,t}^{B}\right]_{x,y}=\mathbbm{1}\left(\left[\boldsymbol{B}_{i,l,t}\right]_{x,y} \geq \tau_{s_{i,l,t}}\right),
  \end{aligned}
\end{equation}
where $\mathbbm{1}_{\left\{x\right\}}=1$ if condition $x$ is true, $\mathbbm{1}_{\left\{x\right\}}=0$, otherwise.
$\tau_{s_{i,l,t}}$ is the threshold corresponding to the top $s_{i,l,t}\%$ value in $\boldsymbol{B}_{i,l,t}$.

Since the projection matrix $\boldsymbol{A}_{i,l}$ are randomly generated on each device and $\boldsymbol{B}_{i,l,t}$ are activated based on their fine-tuned expansion matrix, we assume that at each parameter position, $\boldsymbol{B}_{i,l,t}$ has a fully independent activation probability.
Then, we can define the parameter collision rate during aggregating set of expansion matrix $\{\boldsymbol{B}_{i,l,t} \odot \boldsymbol{M}_{i,l,t}^{B}|i\in \mathcal{M}_{i,t}\}$ which measure the parameter interference during model aggregation as
  \begin{align}
    S_{\mathcal{M}_{i,t},l}=1& \!-\!\!\!\prod\limits_{j\in\mathcal{M}_{i,t}}\!\!\!\left(1\!-\!s_{j,l,t}\right) \!-\!\!\sum\limits_{j\in\mathcal{M}_{i,t}}\!\!s_{j,l,t}\!\!\!\!\prod\limits_{z\in\mathcal{M}_{i,t},z\neq j}\!\!\!\left(1-s_{z,l,t}\right),
  \end{align}
where $S_{\mathcal{M}_{i,t},l}\in\left[0,1\right]$.
We can see that the parameter collision rate $S_{\mathcal{M}_{i,t},l}$ depends on both the devices in the aggregation set $\mathcal{M}_{i,t}$ and their activation probability $s_{i,l,t}$, respectively.
To minimize the parameter interference rate $S_{\mathcal{M}_{i,t},l}$ of devices in $\mathcal{M}_{i,t}$, we need to design the device set and their activation probabilities $s_{i,l,t}$, respectively.

Since the computing and communication resources of each device are limited, the maximum number of parameters that need to be updated and transmitted is restricted as shown in (\ref{eq:max1}c).
At the same time, heterogeneous datasets exhibit diversity across multiple dimensions, leading to varied representational requirements across different layers \cite{THANORA}.
This motivates the design of a task-specific sparsity rate allocation method across layers, which can better match the representational demands of each task and reduce overall parameter collisions across devices in $\mathcal{M}_{i,t}$.
To measure the heterogeneous task's representational demands across layers for sparsity rates allocation, we need to measure these discrepancies.
Let $\boldsymbol{H}^{0}_{i,l,\text{input}}$ be the average input of pre-trained model $\boldsymbol{W}_0$ at $l$-th layer for a small batch of data from $\mathcal{D}_i$,  which is stable and independent of the fine-tuning process.
Then, we introduce the covariance $\boldsymbol{C}_{i,l} = \boldsymbol{H}^{0}_{i,l,\text{input}}\left(\boldsymbol{H}^{0}_{i,l,\text{input}}\right)^{\top}$ and derive a SVD as follows
  \begin{align}
    \widehat{\boldsymbol{W}}_{i,l}=\operatorname{SVD}\left(\boldsymbol{W}_{0,l}\boldsymbol{C}_{i,l}\right) \boldsymbol{C}_{i,l}^{-1}=\sum_{i=1}^{R_{i,l}} \sigma_{i,l} \mathbf{u}_{i,l} {\mathbf{v}}_{i,l}^{\top},
\end{align}
where $\sigma_{i,l}$ is the singular value, $\mathbf{u}_{i,l}$ and ${\mathbf{v}}_{i,l}$ represent corresponding vectors, and 
$R_{i,l}=\text{Rank}\left(\widehat{\boldsymbol{W}}_{i,l}\right)$ is the rank dimension of $\widehat{\boldsymbol{W}}_{i,l}$.
Since the singular value can measure the complexity of the representation required by task in the $l$-th layer, we can compute the spectral entropy of the singular values $\sigma_{i,l}$ for task $i$ in the $l$-th layer as 
\begin{align}
\mathcal{H}\left(\boldsymbol{\sigma}_{i,l}\right)=-\sum_{j=1}^{R} p_{j,l} \log \left(p_{j,l}\right),
\end{align}
where $\boldsymbol{\sigma}_{i,l}=\left[\sigma_{i,1}, \cdots, \sigma_{i,R_{i,l}}\right]$ and $p_{j,l}=\frac{\sigma_{j,l}}{\sum_{z=1}^{R_{i,l}} \sigma_{z,l}}$ is the normalized probability.
Under this fixed total number of trainable parameter budget $Q_{i,\text{budget}}=\frac{c_{i,\text{max}} B}{\varepsilon N_{i}}$ which is restricted by the communication (\ref{eq:max1}a) and computation resources (\ref{eq:max1}c), the sparsity rate of task $i$ in layer $l$ can be allocated by 
\begin{equation}\label{sparity_rate}
  \begin{aligned}
    s_{i,l,t}=\frac{\text{max}[Q_{i,\text{budget}}\times \left[\text{Softmax}\left(\mathcal{H}\left(\boldsymbol{\sigma}_{}\right)\right)\right]_{l},d_l \times r_l]}{d_l \times r_l},
  \end{aligned}
\end{equation}
where $\left[\text{Softmax}\left(\mathcal{H}\left(\boldsymbol{\sigma}_{}\right)\right)\right]_{l}$ is the normalized distribution for $l$-th layer based on Softmax function.
By introducing a fixed total dimension budget and layer-wise importance, we employ a task-aware sparsity rate design method that can accommodate heterogeneous requirements under resource constraints.

\subsection{Cluster-based Device Connection Topology Design based on DFL Convergence Analysis to Solve Problem 2}
To simplify the device connection $\boldsymbol{U}$ optimization in (\ref{eq:max1}), we must analyze the impact of $\boldsymbol{U}$ on the DFL model convergence under the proposed orthogonal $\boldsymbol{A}_{i,l}$, sparse activation $\boldsymbol{B}_{i,l,t}$, and refinement design using task-aware vector first.
Since the each projection matrix $\boldsymbol{A}_{i,l}$ is independently sampled by Gaussian distribution, and each expansion matrix $\boldsymbol{B}_{i,l,t}$ is updated based on independent dataset and sparse mask, each adapter $\boldsymbol{w}_{i,l,t}$ and the masks $\boldsymbol{M}_{i,l,t}$ are independent with each other, we can make the following assumptions:
\begin{itemize}
    \item {\it{Assumption$~1$}:}
    For device $i,j \in \mathcal{M}$, we assume $E||\nabla_{\boldsymbol{B}_{i,l,t}}F\left(\boldsymbol{W}_0, \boldsymbol{w}_{i,t},\mathcal{D}_{i}\right)\boldsymbol{A}_{i,l}||\leq G$.
\item {\it{Assumption$~2$}:}
    For each device $i \in \mathcal{M}$, we assume $E|| \frac{1}{|\mathcal{M}_{i,t}|} \sum\limits_{j\in\mathcal{M}_{i,t}}\boldsymbol{B}_{j,l,t}\boldsymbol{A}_{i,l}||\leq P$.
\end{itemize}
These assumptions are natural, where Assumption~$1$ stems from the fact that the upper bound of the local gradient exists, and Assumption~$2$ assumes an upper bound of the product between Gauss-sampled $\boldsymbol{A}_{i,l}$ and independent updated $\boldsymbol{B}_{j,l,t}$ exists.
Then, we analyze the impact of $\boldsymbol{U}$ on the upper bound of the gap between the local fine-tuned adapter and the aggregated global adapter, which can be given by
\begin{theorem}\label{thm:theorem1}
Given the model transmission matrix $\boldsymbol{U}$, an upper bound of the gap between each local model $\left(\boldsymbol{B}_{i,l,t}\odot\boldsymbol{M}_{i,l,t}\right)\boldsymbol{A}_{i,l,\text{refined}}$ of device $i$ and the average model $\left(\frac{1}{|\mathcal{M}_{i,t}|}\sum_{j\in \mathcal{M}_{i,t}}\boldsymbol{B}_{j,l,t}\odot\boldsymbol{M}_{j,l,t}\right)\boldsymbol{A}_{i,l,\text{refined}}$ of all neighboring devices in $\mathcal{M}_{i,t}=\left\{j|u_{i,j,t}=1,\forall j \in \mathcal{M}\right\}$ can be given by
\begin{equation}
    \begin{aligned}
        & \sum\limits_{i\in\mathcal{M}_{i,t}}\!\!\!\!E||\left(\boldsymbol{B}_{i,l,t}\odot\boldsymbol{M}_{i,l,t}\right)\boldsymbol{A}_{i,l,\text{refined}}-\\
        & \left(\frac{1}{|\mathcal{M}_{i,t}|}\sum_{j\in \mathcal{M}_{i,t}}\boldsymbol{B}_{j,l,t}\odot\boldsymbol{M}_{j,l,t}\right)\boldsymbol{A}_{i,l,\text{refined}}||^{2} \\
        &\leq 2rk\sum\limits_{i=1}^{M} \sum\limits_{j=1}^{M}S_{\mathcal{M}_{i,t},l}\left(G+P\right).
    \end{aligned}
\end{equation}
\end{theorem}
\begin{proof}
    See Appendix A.
\end{proof}
From Theorem \ref{thm:theorem1}, we can see that the gap between each local adapter of device $i$ and the aggregated adapters of all devices in $\mathcal{M}_{i,t}$ is affected by $S_{\mathcal{M}_{i,t},l}$, which depends on the device connection topology $\boldsymbol{U}$.
From Theorem \ref{thm:theorem1}, we can also observe that this gap decreases as the parameter collision rate $S_{\mathcal{M}_{i,t},l}$ decreases, which implies that the device connection topology design that can satisfy the parameter collision-free condition between $\boldsymbol{B}_{i,l,t}$ can effectively enhance the DFL performance.
This is due to the fact that each device updates $\boldsymbol{B}_{i,l,t}\odot\boldsymbol{M}_{i,l,t}$ with activation overlap simultaneously, which will cause a multi-party parameter interference, leading to inconsistencies in the optimization directions, thus reducing the DFL performance.
Furthermore, since the parameter collision rate $S_{\mathcal{M}_{i,t},l}$ grows with the number of devices in $\mathcal{M}_{i,t}$, it is essential to perform aggregation separately.

To enhance DFL performance by reducing $S_{\mathcal{M}_{i,t},l}$, we propose a cluster-based method where devices are grouped into clusters according to their resources and $s_{i,l,t}$, and the aggregation of $\boldsymbol{B}_{i,l,t}$ is performed separately within each cluster (named inner-cluster aggregation), thereby guaranteeing $S_{\mathcal{M}_{i,t},l}$ does not exceed the predefined threshold $S_{\max}$.
After iterations of inner-cluster aggregation, inter-cluster parameter exchange is performed, where parameters exchanged without further aggregation and thus fundamentally expanding the scope of knowledge while avoiding parameter collisions between clusters.
In particular, we employ an AGNES-based cluster algorithm considering devices' computing, communication resources, and parameter collision rate constraints, where each device is initially considered as a single-element cluster (leaf).
At each step of the cluster algorithm, neighbor clusters exchange their communication resources and sparsity rate with each other and find clusters with the lowest parameter collision rate, as well as satisfy the transmission power constraint in (\ref{eq:max1}b), and combine into a new, bigger cluster.
This procedure is iterated until all devices are members of a single, large cluster, or until the predefined parameter collision rate $S_{\max}$ is reached.
The parameter aggregation and exchange in the proposed cluster-based algorithm can be summarized as follows:
\begin{enumerate}
    \item Devices aggregate into different clusters $\mathcal{M}_{i,t} \subseteq \mathcal{C}$ based on cluster algroithm.
    \item Each device $i$ perform local fine-tuning and exchange parameter $\boldsymbol{A}_{i,l}$ and $\boldsymbol{B}_{i,l,t}$ with other devices available in cluster for aggregation.
    \item Each device in cluster $\mathcal{M}_{i,t}$ store the received $\boldsymbol{A}_{i,l}$ and aggregate the received $\boldsymbol{B}_{i,l,t}$ based on (\ref{eq:LoRA_AggregationB}).
    \item After iterations of updating, clusters exchange the stored $\boldsymbol{A}_{i,l}$ and $\boldsymbol{B}_{i,l,t}$ with each other for knowledge sharing.
\end{enumerate}
Steps 2)-4) are performed until the DFL model convergence. 

Hence, each device keeps a group of adapters generated by different clusters $\mathcal{C}=\{\mathcal{M}_{i,t}\}$, the adapter used for the specific input $\boldsymbol{H}_{i,l}$ can be given by 
\begin{equation}
\begin{aligned}
     \boldsymbol{w}_{i,l,t}\boldsymbol{H}^{0}_{i,l,\text{input}}=&\sum\limits_{\mathcal{M}_{i,t}\subseteq\mathcal{C}}\sum\limits_{j\in\mathcal{M}_{i,t}} \left(\frac{1}{|\mathcal{M}_{i,t}|}\sum\limits_{j\in\mathcal{M}_{i,t}} \boldsymbol{B}_{j,l,t}\odot \boldsymbol{M}_{j,t}\right)\\
    & \boldsymbol{A}_{j,l,\text{refined}}\boldsymbol{H}^{0}_{i,l,\text{input}}.
\end{aligned}
\end{equation}
After obtaining the device connection scheme $\boldsymbol{U}$, the optimal transmit power can be given by [Lemma 1, \cite{mypaper}] as follows:
\begin{lemma}
{\rm The optimal transmit power $p_{i,j,t}$ of device $i$ for transmitting its FL model to device $j$ is}
  \begin{equation}\label{eq:L1}
    \begin{split}
      p_{i,j,t}^\ast  = \frac{u_{i,j,t}\sigma_{N}^2}{h_{i,j,t}\left(\boldsymbol{\varepsilon}_{t}\right)}\left(2^\frac{Q_{i,t}\left(\boldsymbol{s}_{i,t}\right) ||\boldsymbol{u}_i||}{W \Gamma} -1\right).
    \end{split}
  \end{equation}
\end{lemma}
\subsection{Task-Aware Coding and Implicit MoE Design to Solve Problem 3}
To address Problem 3 without introducing additional computational overhead from the MoE router, we propose an implicit MoE mechanism that combines task-aware coding with a static, orthogonal-based projection matrix design. 
To minimize the interference term in (\ref{eq:example}), we refined static and orthogonal-based projection matrix $\boldsymbol{A}_{i,l}$ in (\ref{eq:orgA}) with $\boldsymbol{H}^{0}_{i,l,\text{input}}$ which is given by
\begin{equation}
\begin{aligned}
&\boldsymbol{A}_{i,l,\text{refined}} = \boldsymbol{A}_{i,l}\frac{\boldsymbol{H}^{0}_{i,l,\text{input}}}{||\boldsymbol{H}^{0}_{i,l,\text{input}}||},
\end{aligned}
\end{equation}
where $\frac{\boldsymbol{H}^{0}_{i,l,\text{input}}}{||\boldsymbol{H}^{0}_{i,l,\text{input}}||}$ is the normalized latent vector.
By incorporating normalized $\boldsymbol{H}^{0}_{i,l,\text{input}}$ of task $i$ into the projection matrix $\boldsymbol{A}_{i,l}$, the interference term in (\ref{eq:example}) can be given by 
\begin{equation}
    \begin{aligned}\label{eq:interference_term}
        &\left(\boldsymbol{B}_{i,l,t}\odot\boldsymbol{M}^{B}_{j,l,t}\right)\boldsymbol{A}_{i,l,\text{refined}}\boldsymbol{H}_{l,\text{input}} \\
        =&\left(\boldsymbol{B}_{i,l,t}\odot\boldsymbol{M}^{B}_{j,l,t}\right)\boldsymbol{A}_{i,l}\frac{\boldsymbol{H}^{0}_{i,l,\text{input}}\boldsymbol{H}_{l,\text{input}}}{||\boldsymbol{H}^{0}_{i,l,\text{input}}||}.
    \end{aligned}
\end{equation}
The interference term (\ref{eq:interference_term}) scaled down by $\frac{\boldsymbol{H}^{0}_{i,l,\text{input}}\boldsymbol{H}_{l,\text{input}}}{||\boldsymbol{H}^{0}_{i,l,\text{input}}||}$ when  $\boldsymbol{H}_{l,\text{input}}$ has completely different representation with $\boldsymbol{H}^{0}_{i,l,\text{input}}$.
On the other hand, when $\boldsymbol{H}^{0}_{i,l,\text{input}}$ and $\boldsymbol{H}_{l,\text{input}}$ have a highly similar representation, it implies the adapter fine-tuned based on task $i$ and $\boldsymbol{H}_{l,\text{input}}$ equipped with similar knowledge, which can also be scaled up simultaneously.
Hence, the proposed method that refines the projection matrix $\boldsymbol{A}_{i,l}$ with task-aware input vector $\boldsymbol{H}^{0}_{i,l,\text{input}}$ based on the unified pre-trained model $\boldsymbol{W}_{0}$ can re-scale the interference from interference from other unrelated tasks, but also promotes the knowledge integration.
It can also be proved that the proposed orthogonality, sparsity activation, and refinement method does not change the orthogonality between $\boldsymbol{A}_{i,l,\text{refined}}$ as follows:
\begin{theorem}\label{thm:theorem2}(Subspace Orthogonality under Sparsity Activation and Task-Aware Refinement)
Let $\boldsymbol{w}_{i,l,t}=\left(\boldsymbol{B}_{i,l,t}\odot \boldsymbol{M}_{i,l,t}^{B}\right)\boldsymbol{A}_{i,l,\text{refined}}$ and $\boldsymbol{w}_{j,l,t}=\left(\boldsymbol{B}_{j,l,t}\odot \boldsymbol{M}_{j,l,t}^{B}\right)\boldsymbol{A}_{j,l,\text{refined}}$.
The expectation of their product reveals:
\begin{equation}
\begin{aligned}
    \left<\boldsymbol{w}_{i,l,t},\boldsymbol{w}_{j,l,t}^{\top}\right>_F=0.
\end{aligned}
\end{equation}
\end{theorem}
\begin{proof}
    See Appendix B.
\end{proof}
From Theorem $\ref{thm:theorem2}$, we can observe that both the proposed sparse activation (for minimizing interference during fine-tuning) and the refinement method (for minimizing interference during inference) maintain orthogonality between adapters, thereby ensuring that inter-adapter interference is minimized.

Next, to activate related expert after obtaining the relevant coefficient through the matching between specific inputs $\boldsymbol{H}_{l,\text{input}}$ and the refined projection matrix  $\boldsymbol{A}_{i,l,\text{refined}}$, we must filter out irrelevant activation in $\boldsymbol{A}_{i,l,\text{refined}}\boldsymbol{H}_{l,\text{input}}$. 
Thus, we further propose a sparse expert-activated mechanism based on a task-aware vector during inference, where the task type of the input is unknown.
Given an input $\boldsymbol{H}_{l,\text{input}}$ of layer $l$, the output of the proposed top-$k$ activation method is given by 
\begin{equation}
\begin{aligned}
    \boldsymbol{y}=\left(\boldsymbol{B}_{i,l,t}\odot\boldsymbol{M}^{B}_{i,l,t}\right)\text{topk}\left(\boldsymbol{A}_{i,l,\text{refined}}\boldsymbol{H}_{l,\text{input}}\right),
\end{aligned}
\end{equation}
where $\text{topk}\left(\boldsymbol{x}\right)=\boldsymbol{x}\odot\boldsymbol{M}_{k}$ and $\left[{M}_{k}\right]_{x,y}=1$ if $\left[\boldsymbol{x}\right]_{x,y}$ is among the $k$-th largest value in $\boldsymbol{x}$, and $\left[{M}_{k}\right]_{x,y}=0$, otherwise.
By selecting the top-$k$ expert, the task-irrelevant parts of $\boldsymbol{A}_{i,l,\text{refined}}\boldsymbol{H}_{l,\text{input}}$ are effectively suppressed.
It can be seen as a implicit MoE strategy is achieved without an additional explicit MoE router, but rather through implicit task-aware coding.
\section{Simulation Results}
For our simulations, we consider a DFL framework with a circular area and uniformly distributed $M\leq15$ devices across the region, each possessing heterogeneous computing and communication capabilities.
Each device is equipped with the same pre-trained Qwen 2.5-1.5B-Instruct LLM or Qwen 2.5-7B-Instruct LLM, and a fine-tuning dataset that focuses on four capabilities: 
(i) natural language understanding, which includes BoolQ, Piqa, and SocialIqa,
(ii) mathematical reasoning and science, which includes GSM8K, Arc-easy, Arc-challenge, 
(iii) code generation, which includes HumanEval and MBPP,  
and (iv) comprehensive ability, which includes DollyTails, Hellaswag, and ScienceQA.
For comparison, we utilize five baselines :
\begin{itemize}
    \item A decentralized fine-tuning method based on parameter orthogonalization and a sparsification method in LoRI \cite{LoRI}, where devices generate a random device connection under communication and computing resources constraints (labelled "LoRI" in plots).
    \item A decentralized fine-tuning based on traditional LoRA \cite{LoRA} (without parameter orthogonalization and sparsification) and a random device connection under communication resources constraints (labelled "LoRA" in plots).
    \item A hard-routing MoE mechanism, where a perfect MoE router is introduced that can assigned task to the related experts that are fine-tuned based on the specific task dataset accordingly (labelled "Hard-routing MoE" in plots).
    \item A sparse-and-orthogonal with implicit MoE mechanism, which has the same set as the proposed method and devices are allocated into the same cluster (labelled "Proposed (Single Cluster)" in plots).
    \item A full parameter fine-tuning (FPFT) method (labelled "FPFT" in the plots).
\end{itemize}
\begin{figure}[t]
  \centering
  \includegraphics[width=8.5cm]{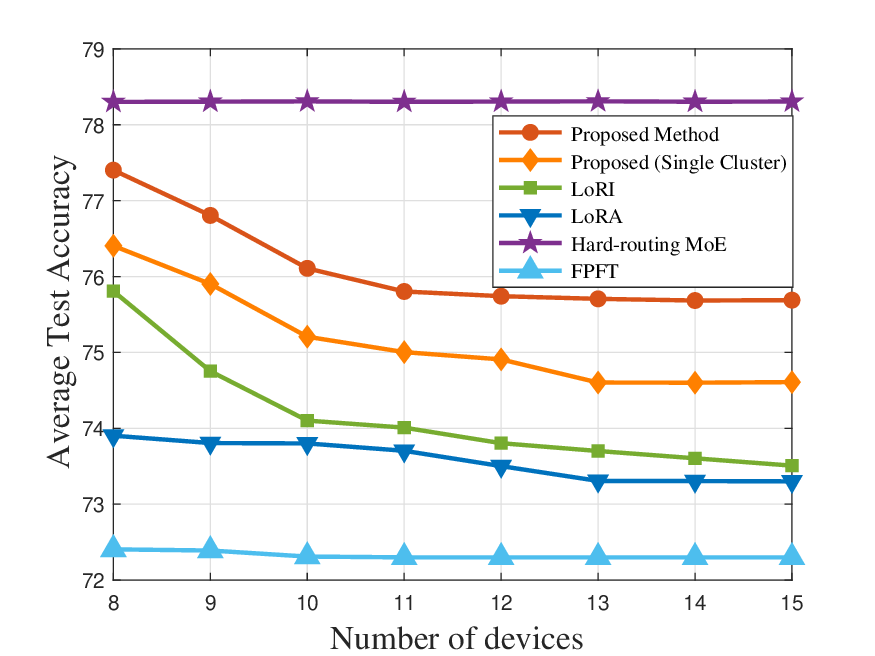}
  \centering
  \caption{Average Parameter Collision Rate vs. the Number of Devices.}
  \vspace{-0.4cm}
  \label{Perf.Num}
\end{figure}
\begin{figure}[t]
  \centering
  \includegraphics[width=8.5cm]{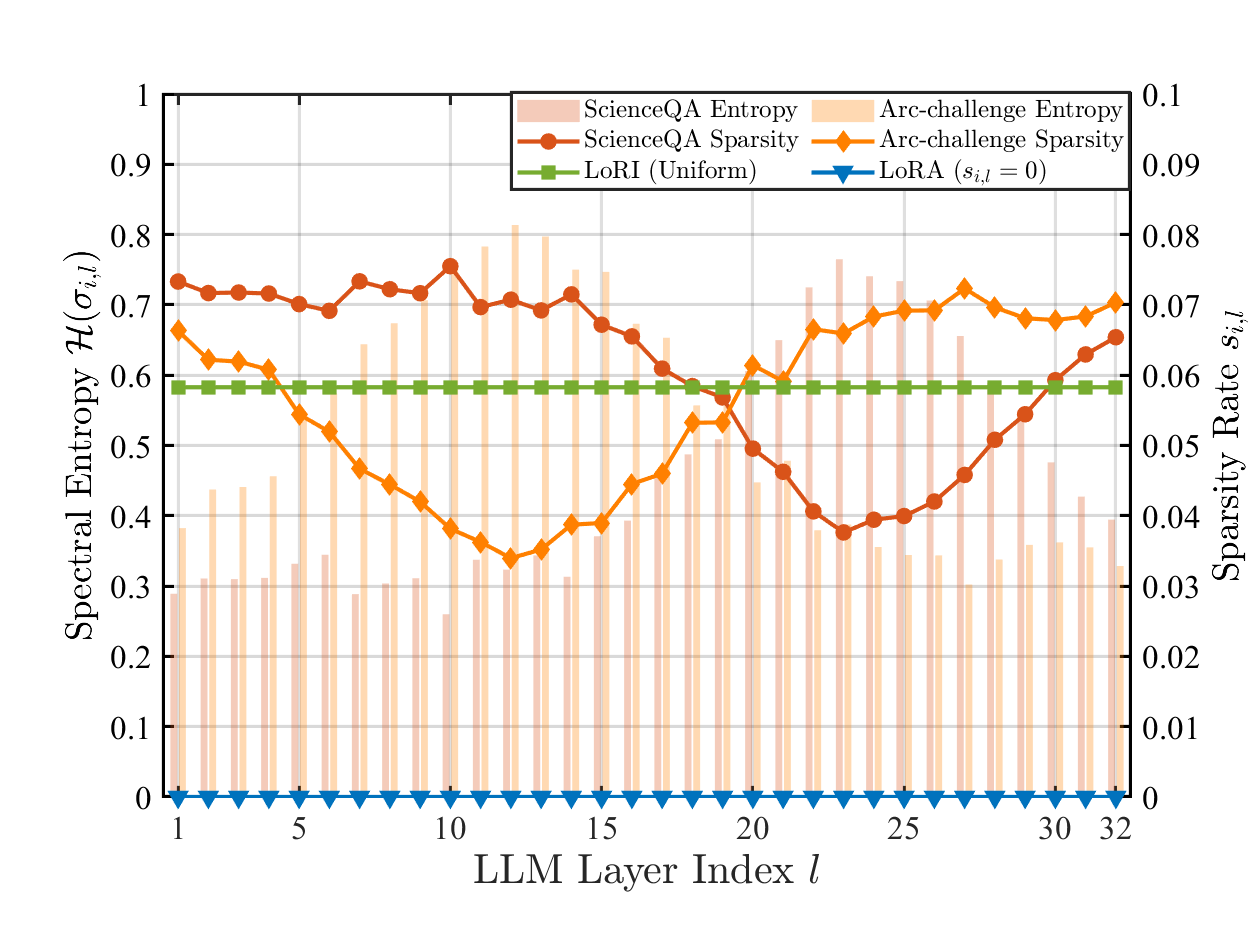}
  \centering
  \caption{Spectral Entropy and Sparsity Rate vs. Layers.}
  \vspace{-0.4cm}
  \label{Spar.Layers}
\end{figure}
\begin{figure}[t]
  \centering
  \includegraphics[width=8.5cm]{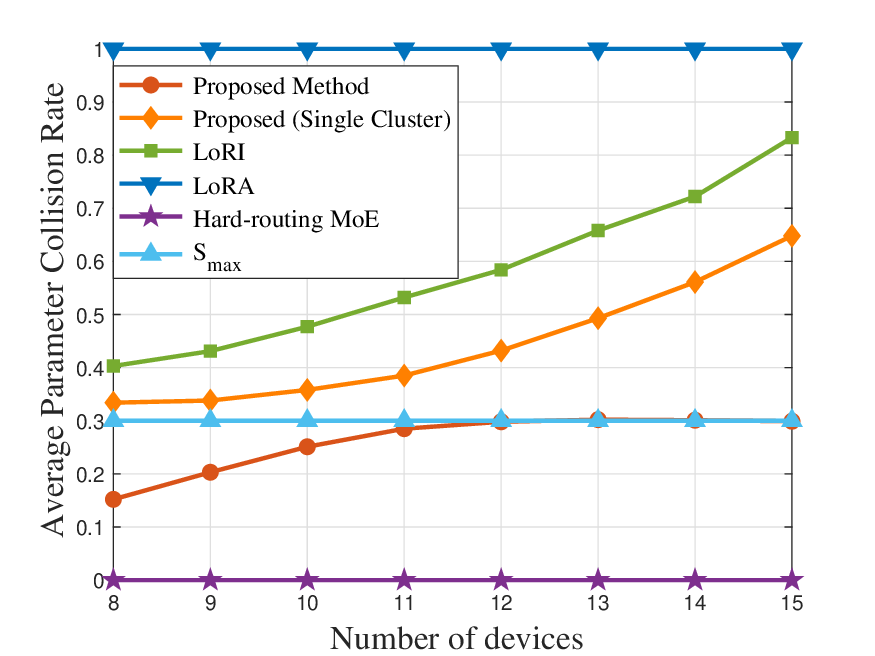}
  \centering
  \caption{Test accuracy vs. the Number of Devices.}
  \vspace{-0.4cm}
  \label{CollisionRate}
\end{figure}

\begin{table*}[htbp]
\centering
\caption{Comparison of different fine-tuning methods using Qwen2.5 models across various benchmarks with $r_l = 32$ in DFL setting. Bold indicates the best-performing method, and underline indicates the second-best.
}
\label{tab:Multiresults}
\small
\setlength{\tabcolsep}{4pt} 
\begin{tabular}{l | l | l | c c c c c c c c | r}
\toprule
\textbf{Method} & \textbf{\# Params (\%)} & \textbf{\# Trans. (\%)} & \textbf{BoolQ} & \textbf{Piqa} & \textbf{SocI.qa} & \textbf{GSM8K} & \textbf{ARC-E.} & \textbf{ARC-C.} & \textbf{Hella.S.} & \textbf{HumanE.} & \textbf{Avg.} \\
\midrule
\rowcolor[HTML]{E8F1FF} 
\multicolumn{12}{l}{\textit{Qwen2.5-1.5B-Instruct}} \\
FPFT & 3.2G (100\%) & 3.2G (100\%) & 79.81 & 80.72 & 79.16 & 77.83 & 84.49 & 64.2 & 87.35 & 62.05 & 76.95 \\
LoRA & 65.5M (2.01\%) & 65.5M (2.01\%) & 76.94 & {78.59} & {79.20} & 71.90 & 62.30 & {52.00} & {68.30} & {47.90} & {68.03} \\
LoRI & 32M (1.05\%) & 32M (1.05\%) & 79.29 & {81.48} & {78.64} & 74.12 & {82.78} & 62.11 & 84.4 & {58.58} & {75.17} \\
Hard-routing MoE & 3.2G (100\%) & 3.2G (100\%) & \textbf{83.23} & \textbf{84.87} & \textbf{83.04} & \textbf{82.44} & \textbf{87.56} & \textbf{69.12} & \textbf{90.91} & \textbf{66.29} & \textbf{80.93} \\
\rowcolor[HTML]{EFEFEF} 
Proposed (Single) & \textbf{16M (0.54\%)} & \textbf{8M (0.27\%)} & 80.21 & {81.69} & {78.38} & {78.21} & {82.95} & {64.25} & 84.74 & {61.09} & 76.44 \\
\rowcolor[HTML]{EFEFEF} 
Proposed Method & \textbf{16M (0.54\%)} & \textbf{8M (0.27\%)} & \underline{81.45} & \underline{83.36} & \underline{79.49} & \underline{80.04} & \underline{84.41} & \underline{65.97} & \underline{86.09} & \underline{62.67} & \underline{78.00} \\
\midrule
\rowcolor[HTML]{E8F1FF} 
\multicolumn{12}{l}{\textit{Qwen2.5-7B-Instruct}} \\
FPFT & 14.2G (100\%) & 14.2G (100\%) & 84.25 & 85.65 & 83.96 & 92.57 & 93.31 & 73.29 & 91.8 & 87.57 & 86.55 \\
LoRA & 362M (2.5\%) & 362M (2.5\%) & 81.32 & {83.88} & 80.99 & {88.56} & 89.38 & 70.25 & {88.18} & 82.95 & 83.19 \\
LoRI & {168M (1.16\%)} & {168M (1.16\%)} & 82.72 & {83.22} & 83.4 & {90.33} & 90.82 & {70.56} & 89.12 & {84.67} & 84.36 \\
Hard-routing MoE & 14.2G (100\%) & 14.2G (100\%) & \textbf{87.67} & \textbf{90.46} & \textbf{87.11} & \textbf{96.84} & \textbf{96.95} & \textbf{78.21} & \textbf{95.58} & \textbf{92.13} & \textbf{90.62} \\
\rowcolor[HTML]{EFEFEF} 
Proposed (Single) & \textbf{87M (0.60\%)} & \textbf{42M (0.30\%)} & 83.33 & {85.74} & {83.74} & 91.14 & 90.69 & {75.17} & 89.78 & {88.08} & 85.96 \\
\rowcolor[HTML]{EFEFEF} 
Proposed Method & \textbf{87M (0.60\%)} & \textbf{42M (0.30\%)} & \underline{84.76} & \underline{86.86} & \underline{85.52} & \underline{92.49} & \underline{92.6} & \underline{76.23} & \underline{91.32} & \underline{89.35} & \underline{87.39} \\
\bottomrule
\end{tabular}
\end{table*}
\begin{table*}[htbp]
\centering
\caption{Comparison of different fine-tuning methods using Qwen2.5 models across various benchmarks with $r_l = 32$ in single task.
Bold indicates the best-performing method, and underline indicates the second-best.}
\label{tab:Singleresults}
\small
\setlength{\tabcolsep}{4pt} 
\begin{tabular}{l | l | c c c c c c c c | r}
\toprule
\textbf{Method} & \textbf{\# Params (\%)} & \textbf{BoolQ} & \textbf{Piqa} & \textbf{SocI.qa} & \textbf{GSM8K} & \textbf{ARC-E.} & \textbf{ARC-C.} & \textbf{Hella.S.} & \textbf{HumanE.} & \textbf{Avg.} \\
\midrule
\rowcolor[HTML]{E8F1FF} 
\multicolumn{11}{l}{\textit{Qwen2.5-1.5B-Instruct}} \\
FPFT & 3.2G (100\%) & \textbf{83.23} & \textbf{84.87} & \textbf{83.04} & \textbf{82.44} & \textbf{87.56} & \textbf{69.12} & \textbf{90.91} & \textbf{66.29} & \textbf{80.93} \\
LoRA & 65.5M (2.01\%) & 80.09 & 83.65 & 79.33 & 75.78 & 84.41 & 63.38 & 85.15 & 60.5 & {76.53} \\
LoRI & 32M (1.05\%) & {81.72} & {84.19} & {80.82} & {77.07} & {85.14} & {64.93} & {86.47} & {61.17} & {77.69} \\
\rowcolor[HTML]{EFEFEF} 
Proposed Method & \textbf{16M (0.54\%)} & \underline{82.88} & \underline{84.54} & \underline{81.25} & \underline{81.36} & \underline{86.30} & \underline{67.02} & \underline{87.63} & \underline{63.94} & \underline{79.37} \\
\midrule
\rowcolor[HTML]{E8F1FF} 
\multicolumn{11}{l}{\textit{Qwen2.5-7B-Instruct}} \\
FPFT & 14.2G (100\%) & \textbf{87.67} & \textbf{90.46} & \textbf{87.11} & \textbf{96.84} & \textbf{96.95} & \textbf{78.21} & \textbf{95.58} & \textbf{92.13} & \textbf{90.62} \\
LoRA & 362M (2.5\%) & 84.54 & 87.73 & 84.40 & 92.32 & 92.52 & 74.18 & 91.76 & 86.22 & 86.71 \\
LoRI & 168M (1.16\%) & {85.15} & {87.03} & {85.55} & {93.89} & {93.60} & {74.48} & {92.26} & {87.34} & {87.41} \\
\rowcolor[HTML]{EFEFEF} 
Proposed Method & \textbf{87M (0.06\%)} & \underline{86.18} & \underline{88.71} & \underline{86.79} & \underline{94.08} & \underline{93.73} & \underline{77.97} & \underline{93.23} & \underline{90.71} & \underline{88.93} \\
\bottomrule
\end{tabular}
\end{table*}

In Fig. \ref{Perf.Num}, we show how the average test accuracy changes as the number of devices varies. 
From Fig. \ref{Perf.Num} we can also observe that the hard-routing MoE and FPFT have the highest and lowest average performance, respectively, and their performance does not change with the number of devices.
This is due to the fact that the hard-routing MoE mechanism employs a perfect MoE-router that can assign the tasks to the most related expert, which can effectively prevent catastrophic knowledge forgetting and multi-task knowledge interference.
On the other hand, FPFT suffers from the heavily catastrophic knowledge forgetting and multi-task knowledge interference since parameters sub-space direction and parameter overlap with each other.
From Fig. \ref{Perf.Num}, we can observe that the average test accuracy of the considered methods decreases as the number of devices increases, except for the hard-routing MoE and FPFT method.
This is due to the fact that catastrophic knowledge forgetting and multi-task knowledge interference occur when more devices participate in the parameter aggregation process, which leads to an average accuracy drop.
From Fig. \ref{Perf.Num}, we can see that the proposed method achieves a $1.3\%$ performance increase compared with the proposed method (single cluster).
This is because the one cluster aggregation may causes unaccaptable parameter collision rate, which also leads to catastrophic knowledge forgetting and multi-task knowledge interference and decreases collaborative fine-tuning performance.
Similarly, the proposed method exceeds the LoRA and LoRI methods due to its orthogonal projection matrix design and layer-wise sparsity expansion matrix, respectively.


In Fig. \ref{Spar.Layers}, we show how the spectral entropy and sparsity rate change across layers in Qwen2.5-7B-Instruct.
Specifically, for simplicity, we only present results for the ScienceQA and Arc-challenge tasks.
As shown in Fig. \ref{Spar.Layers}, the spectral entropy of the ScienceQA and Arc-challenge tasks exhibits a complementary and alternating pattern across layers, rather than following a synchronized trend.
Specifically, in the initial layers, the ScienceQA task exhibits lower spectral entropy compared to the higher values of the Arc-challenge task; this relationship is reversed in the final layers, where the two tasks mirror each other's earlier behaviour.
This implies the difference between tasks' requirements across layers, and leaves room for adaptive sparsity rate allocation accordingly.
From Fig. \ref{Spar.Layers}, we can also observe that the proposed method reallocates the sparsity rate of these tasks across layers based on their spectral entropy distribution.
On the other hand, he considered baseline LoRA and LoRI employ a uniform and non-sparse strategy across layers, respectively, which caused a mismatch between sthe parsity rate and the task requirements of different layers.
Similar to the observation from Fig.\ref{Perf.Num}, which indicates that our proposed method reallocates the sparsity rate across layers, can achieve a lower parameter collision rate and higher performance.

In Fig. \ref{CollisionRate}, we show how the average parameter collision rate changes as the number of devices varies. 
From Fig. \ref{CollisionRate}, we can observe that the proposed method achieves the lowest average parameter collision rate compared with the LoRI and LoRA method, which is strictly constrained by the predefined threshold $S_{\max}$, whereas the other baselines exhibit higher parameter collision rates.
This is due to the fact that the proposed algorithm can reduce the parameter collision rate through cluster-based DFL device connection topology design, which considers each device's sparsity rate and the number of devices in the cluster.
In contrast, the average parameter collision rate of LoRI increases as the number of devices participating in the DFL increases, while LoRA, without parameter sparsity, consistently experiences a $100\%$ parameter collision rate.
Due to the increased parameter collision rate between devices, the corresponding performance is shown in Fig. \ref{Perf.Num}. 
As the average parameter collision rate of LoRI increases as the number of devices participating in the DFL increases, while LoRA, lacking parameter sparsity, consistently experiences a $100\%$ parameter collision rate.
Taking both Fig. \ref{Perf.Num} and Fig. \ref{CollisionRate} into consideration, we can see that the proposed cluster method can significantly improve the performance of the static-and-orthogonal projection matrix and layer-wise sparsity expansion matrix design by comparing the performance between the proposed method and the proposed (single cluster) scheme.
We can also observe from both Fig. \ref{Perf.Num} and Fig. \ref{CollisionRate} that the cluster method and subspace orthogonal projection matrix can effectively reduce parameter interference and enable convergence of fine-tuning.

In Tab. \ref{tab:Multiresults} and Tab. \ref{tab:Singleresults}, we show the detailed multi-task DFL average performance and single-task fine-tuning performance without DFL, respectively.
In particular, we show the detailed performance of each task that contains $10$ devices fine-tuning on different strategies in the DFL scenario in Tab. \ref{tab:Multiresults}.
From Tab. \ref{tab:Multiresults}, we can observe that the proposed method can reduce the number of parameters that need to be updated compared to the traditional LoRA method and FPFT method by up to $73\%$ and $99.46\%$, respectively.
We can also observe that the proposed method can reduce the number of parameters that need to be transmitted compared to the traditional LoRA method and FPFT method by up to $86.56\%$ and $99.7\%$, respectively.
This is due to the fact that in the proposed method, $\boldsymbol{A}_{i,l}$ is kept static during fine-tuning, which does not need to be updated and transmitted, which can reduce both the computing and communication overhead.
From Tab. \ref{tab:Singleresults}, we can observe that despite a lower computational budget, the proposed method outperforms traditional LoRA and LoRI baseline variants with the same rank across all datasets.
The outperformance of the LoRA is due to the static projection matrix, which will not compromise accuracy but also achieve a more stable convergence than LoRA, which has twice the number of parameters to be fine-tuned.
It can also prove that compared to the broader parameter space of the pre-trained model, the fixed parameter sub-spaces projected by fixed $\boldsymbol{A}_{i,l}$ are affordable for the specific task.
Besides, the fixed $\boldsymbol{A}_{i,l}$ with a lower degree of freedom can achieve a stable convergence and higher performance compared to the traditional methods.
The outperformance of the proposed method compared to the LoRI method can be attributed to the layer-wise sparsity expansion design.
In particular, the layer-wise sparsity expansion design can achieve an efficient sparsity assignment across layers, thus makes an task-orient sparsity scheme that can achieve a better performance.
\section{Conclusion}
In this paper, we developed a novel DFL framework that enables devices to collaboratively fine-tune an LLM via exchanging their local updates based on heterogeneous datasets with their neighboring device without reliance on a centralized server.
We formulated an optimization problem whose goal is to minimize fine-tuning loss across datasets while meeting the delay, transmit power, and computing power resource constraints in each DFL round.
Then, we identified three key problems around the life-cycle of DFL: (i) catastrophic forgetting during fine-tuning, (ii) communication inefficiencies during aggregation, and (iii) multi-task interference during inference.
To solve these problems, we first introduced a static-and-orthogonal along with a layer-wise sparsity LoRA method to ensure update orthogonality and reduce parameter overlap while reducing the communication overhead.
Then, we designed a cluster-based device connection topology based on the analysis of the impact of device connection and multi-task performance.
Finally, we proposed a task-aware coding and implicit MoE mechanism that can activate the specified expert and filter the unrelated parameters during inference.
Simulation results demonstrated that the proposed algorithm can achieve robust DFL compared to traditional LoRA and LoRI methods, while ensuring lower computational and communication resource consumption.
For future work, we will explore more generalized scenarios where local fine-tuning is unavailable. 
In such cases, lightweight techniques, specifically textual inversion and collaborative inference, will be central to our problem formulation.

\section{Appendix}
\subsection{Proof of Theorem 1}
Under assumptions, we have
\begin{equation}
        \begin{aligned}\label{proof_1}
             & \sum\limits_{i\in\mathcal{M}_{i,t}} E||\boldsymbol{A}_{i,l}\boldsymbol{B}_{i,l,t}\odot\boldsymbol{M}_{i,t}\!\!-\!\!\frac{1}{|\mathcal{M}_{i,t}|}\boldsymbol{A}_{i,l}\!\!\sum\limits_{j\in\mathcal{M}_{i,t}}\boldsymbol{B}_{j,l,t}\!\!\odot\!\!\boldsymbol{M}_{j,t}||^{2} \\
             \leq & \!\! \sum\limits_{i\in\mathcal{M}_{i,t}} \!\! 2 E||\boldsymbol{A}_{i,l}\boldsymbol{B}_{i,l,t}\odot\boldsymbol{M}_{i,t}\!-\!\frac{1}{|\mathcal{M}_{i,t}|} \boldsymbol{A}_{i,l} \!\!\sum\limits_{j\in\mathcal{M}_{i,t}}\boldsymbol{B}_{j,l,t} \!\!\odot \!\!\boldsymbol{M}_{i,t}|| \\
             & + \sum\limits_{i\in\mathcal{M}_{i,t}} 2 E||\frac{1}{|\mathcal{M}_{i,t}|} \boldsymbol{A}_{i,l} \sum\limits_{j\in\mathcal{M}_{i,t}}\boldsymbol{B}_{j,l,t} \odot \boldsymbol{M}_{i,t} \\
             & - \frac{1}{|\mathcal{M}_{i,t}|} \boldsymbol{A}_{i,l} \sum\limits_{j\in\mathcal{M}_{i,t}}\boldsymbol{B}_{j,l,t}\odot\boldsymbol{M}_{j,t}||.
    \end{aligned}
\end{equation}
For the first term in (\ref{proof_1}), we have
\begin{equation}
        \begin{aligned}
             & \sum\limits_{i\in\mathcal{M}_{i,t}} E||\boldsymbol{A}_{i,l}\boldsymbol{B}_{i,l,t}\odot\boldsymbol{M}_{i,t}\!\!-\!\!\frac{1}{|\mathcal{M}_{i,t}|} \boldsymbol{A}_{i,l} \!\!\!\!\sum\limits_{j\in\mathcal{M}_{i,t}}\!\!\boldsymbol{B}_{j,l,t} \odot \boldsymbol{M}_{i,t}|| \\
             = & \sum\limits_{i\in\mathcal{M}_{i,t}} E||\boldsymbol{A}_{i,l}\frac{1}{|\mathcal{M}_{i,t}|}\sum\limits_{j\in\mathcal{M}_{i,t}}\left(\boldsymbol{B}_{i,l,t}-\boldsymbol{B}_{j,l,t}\right) \odot ( \boldsymbol{M}_{j,t} \odot \boldsymbol{M}_{i,t} \\
             & + \left|\boldsymbol{M}_{j,t}-\boldsymbol{M}_{i,t}\right|\odot\boldsymbol{M}_{i,t} )|| \\
             = & \sum\limits_{i\in\mathcal{M}_{i,t}} E||\boldsymbol{A}_{i,l}\frac{1}{|\mathcal{M}_{i,t}|}\sum\limits_{j\in\mathcal{M}_{i,t}}\left(\boldsymbol{B}_{i,l,t}-\boldsymbol{B}_{j,l,t}\right) \odot \boldsymbol{M}_{i,j,t}^{c}||.
    \end{aligned}\label{proof_2}
\end{equation}
Due to the definition of parameter collision rate, (\ref{proof_2}) can be further given by
\begin{equation}
        \begin{aligned}
            & rk \sum\limits_{i\in\mathcal{M}_{i,t}} \!\!\!\sum\limits_{j\in\mathcal{M}_{i,t}}S_{\mathcal{M}_{i,t},l}
             E||\boldsymbol{A}_{i,l}\left(\boldsymbol{B}_{i,l,t}\!-\!\frac{1}{|\mathcal{M}_{i,t}|}\sum\limits_{j\in\mathcal{M}_{i,t}}\!\!\!\boldsymbol{B}_{j,l,t}\right)|| \\
             = & rk \sum\limits_{i\in\mathcal{M}_{i,t}} \!\!\!\sum\limits_{j\in\mathcal{M}_{i,t}}S_{\mathcal{M}_{i,t},l}
             E|| - \eta\boldsymbol{A}_{i,l}\nabla_{\boldsymbol{B}_{i,l,t}}F_i||\\
             \leq & rk \sum\limits_{i\in\mathcal{M}_{i,t}} \!\!\!\sum\limits_{j\in\mathcal{M}_{i,t}}S_{\mathcal{M}_{i,t},l}G.
             \end{aligned}
\end{equation}
The second equation is due to the fact that $\boldsymbol{A}_{i,l}$ is independent to $\frac{1}{|\mathcal{M}_{i,t}|}\sum\limits_{j\in\mathcal{M}_{i,t}}\left(\boldsymbol{B}_{i,l,t}-\boldsymbol{B}_{j,l,t}\right)\left|\boldsymbol{M}_{j,t}-\boldsymbol{M}_{i,t}\right|\odot\boldsymbol{M}_{i,t}$.
Similarly, we can rewrite the second term in (\ref{proof_1}) as
\begin{equation}
        \begin{aligned}
             & \sum\limits_{i\in\mathcal{M}_{i,t}}\!\!\! E||\!\frac{1}{|\mathcal{M}_{i,t}|} \boldsymbol{A}_{i,l}\!\!\! \sum\limits_{j\in\mathcal{M}_{i,t}}\!\!\!\boldsymbol{B}_{j,l,t} \odot \boldsymbol{M}_{j,t} \\
             & -\!\frac{1}{|\mathcal{M}_{i,t}|} \boldsymbol{A}_{i,l} \!\!\!\!\sum\limits_{j\in\mathcal{M}_{i,t}}\!\!\!\boldsymbol{B}_{j,l,t}\odot\boldsymbol{M}_{i,t}|| \leq rk \sum\limits_{i\in\mathcal{M}_{i,t}} S_{\mathcal{M}_{i, t}} P.
    \end{aligned}
\end{equation}
Thus, we have 
\begin{equation}
    \begin{aligned}
    & \sum\limits_{i\in\mathcal{M}_{i,t}} E||\boldsymbol{A}_{i,l}\boldsymbol{B}_{i,l,t}\odot\boldsymbol{M}_{i,t}\!-\!\frac{1}{|\mathcal{M}_{i,t}|}\boldsymbol{A}_{i,l}\!\!\sum\limits_{i\in\mathcal{M}_{i,t}}\!\!\boldsymbol{B}_{j,l,t}\odot\boldsymbol{M}_{j,t}||^{2} \\
    & \leq 2rk \sum\limits_{i\in\mathcal{M}_{i,t}} \sum\limits_{j\in\mathcal{M}_{i,t}}S_{\mathcal{M}_{i,t},l}\left(G+P\right).
    \end{aligned}
\end{equation}
This ends the proof. 

\subsection{Proof of Theorem 2}
By the definition of the Frobenius inner product:
\begin{equation}
        \begin{aligned}
            &\left<\boldsymbol{w}_{i,l,t},\boldsymbol{w}_{j,l,t}^{\top}\right>_F=\text{Tr}\left(\boldsymbol{w}_{i,l,t}^{\top},\boldsymbol{w}_{j,l,t}\right)\\
            =&\text{Tr}\left(\left(\boldsymbol{B}_{j,l,t}\odot \boldsymbol{M}_{j,l,t}^{B}\right)\boldsymbol{A}_{j,l,\text{refined}}\boldsymbol{A}_{i,l,\text{refined}}^{\top}\left(\boldsymbol{B}_{i,l,t}\odot \boldsymbol{M}_{i,l,t}^{B}\right)^{\top}\right)\\
            =&\text{Tr}\!\left(\!\!\left(\boldsymbol{B}_{j,l,t}\!\odot \!\boldsymbol{M}_{j,l,t}^{B}\right)\!\!
            \frac{\left[\boldsymbol{H}^{0}_{j,l}\right]^{\top}}{||\boldsymbol{H}^{0}_{j,l}||}\boldsymbol{A}_{j,l}^{\top}\boldsymbol{A}_{i,l}\!\!\frac{\boldsymbol{H}^{0}_{i,l}}{||\boldsymbol{H}^{0}_{i,l}||}\!\left(\boldsymbol{B}_{i,l,t}\!\odot \!\boldsymbol{M}_{i,l,t}^{B}\right)^{\top}\!\!\!\right)\\
            =& \text{Tr}\!\left(\!\!\!\left(\boldsymbol{B}_{j,l,t}\!\odot\! \boldsymbol{M}_{j,l,t}^{B}\right)\!\!\frac{\left[\boldsymbol{H}^{0}_{j,l}\right]^{\top}}{||\boldsymbol{H}^{0}_{j,l}||}\!\!\left(\!\!\!\boldsymbol{A}_{i,l}
            \boldsymbol{A}_{j,l}^{\top}\!\right)\!\!\frac{\boldsymbol{H}^{0}_{i,l}}{||\boldsymbol{H}^{0}_{i,l}||}\!\!\left(\boldsymbol{B}_{i,l,t}\!\!\odot \!\!\boldsymbol{M}_{i,l,t}^{B}\right)^{\top}\!\!\!\right),
        \end{aligned}
\end{equation}
where $\boldsymbol{H}^{0}_{i,l}$ is the abbreviation of $\boldsymbol{H}^{0}_{i,l,\text{input}}$.
Since elements in $\boldsymbol{A}_{i,l}$ are random generated from Gauss distribution, $\boldsymbol{A}_{i,l} \boldsymbol{A}_{j,l}^{\top}=0$.
Then, we can prove that $\left<\boldsymbol{w}_{i,l,t},\boldsymbol{w}_{j,l,t}^{\top}\right>_F=\text{Tr}\left(\boldsymbol{w}_{i,l,t}^{\top},\boldsymbol{w}_{j,l,t}\right)=0$.

This completes the proof.

\end{document}